\documentclass{article}

    \usepackage[preprint, nonatbib]{neurips_2023_arXiv}

\usepackage[utf8]{inputenc} 
\usepackage[T1]{fontenc}    
\usepackage{hyperref}       
\usepackage{url}            
\usepackage{booktabs}       
\usepackage{amsfonts}       
\usepackage{nicefrac}       
\usepackage{microtype}      
\usepackage{xcolor}         

\usepackage{amsmath,amssymb,mathrsfs}
\usepackage{algorithmic}
\usepackage{algorithm}

\usepackage[non-sorted-cites,initials,alphabetic,nobysame]{amsrefs}

\usepackage{amsthm}

\usepackage{cleveref}

\crefname{equation}{equation}{equations}
\crefname{figure}{figure}{figures}
\crefname{table}{table}{tables}
\crefname{algorithm}{Algorithm}{Algorithms}

\crefname{section}{}{}
\creflabelformat{section}{\S #2#1#3}
\crefname{subsection}{}{}
\creflabelformat{subsection}{\S #2#1#3}
\crefname{appendix}{Appendix}{Appendixes}

\crefname{dammy}{}{}
\crefname{definition}{Definition}{Definitions}
\crefname{proposition}{Proposition}{Propositions}
\crefname{lemma}{Lemma}{Lemmas}
\crefname{theorem}{Theorem}{Theorems}
\crefname{corollary}{Corollary}{Corollaries}
\crefname{remark}{Remark}{Remarks}
\crefname{example}{Example}{Examples}

\theoremstyle{definition}
\newtheorem{dammy}{Dammy}[section]
\newtheorem{definition}[dammy]{Definition}
\newtheorem{proposition}[dammy]{Proposition}
\newtheorem{lemma}[dammy]{Lemma}
\newtheorem{theorem}[dammy]{Theorem}

\newtheorem{remark}[dammy]{Remark}
\newtheorem{example}[dammy]{Example}

\numberwithin{equation}{section}

\DeclareMathOperator*{\argmin}{arg\,min} 
\DeclareMathOperator*{\argmax}{arg\,max} 

\DeclareMathOperator{\Id}{Id}

\DeclareMathOperator{\Tr}{Tr}
\DeclareMathOperator{\Proj}{Proj}

\title{A proof of convergence \\ of inverse reinforcement learning \\ for multi-objective optimization}

\author{%
    Akira Kitaoka \\
    NEC Corporation \\
    \texttt{akira-kitaoka@nec.com}
    \And
    Riki Eto \\
    NEC Corporation \\
    \texttt{riki.eto@nec.com} \\
}

\begin{document}

\maketitle

\begin{abstract}
    We show the convergence of Wasserstein inverse reinforcement learning for multi-objective optimizations with the projective subgradient method by formulating an inverse problem of the multi-objective optimization problem.

    In addition, we prove convergence of inverse reinforcement learning (maximum entropy inverse reinforcement learning, guided cost learning) with gradient descent and the projective subgradient method.
\end{abstract}


\section{Introduction}

Artificial intelligence (AI) has been used to automate various tasks recently. Generally,  automation by AI is achieved by setting an index of goodness or badness (reward function) of a target task and having AI automatically search for a decision, that is, an optimal solution in mathematical optimization that maximizes or minimizes the index. 
For example, in work shift scheduling (e.g. \cites{Cheang-2003-nurse,Graham-1979-optimization}), which is a type of combinatorial optimization or multi-objective optimization, we can create shifts that reflect our viewpoints by calculating the optimal solution of a reward function that reflects our intentions for several viewpoints, such as ``degree of reflection of vacation requests,'' ``leveling of workload,'' and ``personnel training,'' and so on while preserving the required number of workers, required skills, labor rules.
However, setting the reward function, i.e., "what is optimal?", manually requires a lot of trial-and-error, which is a challenge for the actual application of mathematical optimization.
Creating a system that can solve this problem automatically is essential in freeing the user from manually designing the reward function.

Inverse reinforcement learning (IRL) 
\cites{Russell-1998-learning,Ng-2000-algorithms} is generally known as facilitating the setting of the reward function. In IRL, a reward function that reflects expert's intention is generated by learning expert's trajectories, iterating optimization using the reward function, and updating the parameters of the reward function.
In IRLs which is fomulated by Ng and Russell \cite{Ng-2000-algorithms}, and Abbeel and Ng \cite{Abbeel-2004-apprenticeship}, in multi-objective optimization, the space of actions, i.e., the space of optimization results, is enormous.
In other words, it is necessary to set the reward function for the space of actions and states, which is computationally expensive.

Maximum entropy IRL (MEIRL) \cite{Ziebart-2008-maximum}
and 
guided cost learning (GCL) \cite{Finn-2016-guided}
are 
methods to adapt IRL to multi-objective optimization problems.
However, these methods have their issues.
For example, MEIRL requires the sum of the reward functions for all trajectories to be computed. This makes maximum entropy IRL computationally expensive. On the other hand, GCL approximates the sum of the reward functions for all trajectories by importance sampling. However, since multi-objective optimization problems take discrete values, it is difficult to find the probability distribution corresponding to a given value when a specific value is input. One reason for this difficulty is that in multi-objective optimization problems, even a small change in the value of the reward function may result in a large change in the result.

Eto proposed IRL for multi-objective optimization, Wasserstein inverse reinforcement learning (WIRL) \cite{Eto-2022}, inspired by Wasserstein generative adversarial networks (WGAN) \cite{Arjovsky-Chintala-Bottou-17}. Experiments have confirmed that WIRL can stably perform inverse reinforcement learning in multi-objective optimization problems.

However, there was no theoretical way to guarantee that WIRL would converge. 
Eto proposed that we do inverse reinforcement learning for multi-objective optimizations with WIRL \cite{Eto-2022}, although there was no way to confirm this theoretically.
Also, there exists an example in which WGAN does not converge \cite{Mescheder-2018-which}. Therefore, WIRL is not expected to converge in general.

The inverse reinforcement learning algorithm formulated by Ng and Russell \cite{Ng-2000-algorithms} is guaranteed to converge to the optimal solution because of the linear programming.
The inverse reinforcement learning algorithm formulated by Abeell and Ng \cite{Abbeel-2004-apprenticeship} is guaranteed to converge the learner's value function to the expert's value function.
However, the convergence of MEIRL and GCL was not proven.

In this paper, we show the convergence of various inverse reinforcement learning for multi-objective optimizations with the projective subgradient method. 
In particular, We show the convergence of Wasserstein inverse reinforcement learning (WIRL) for multi-objective optimizations with the projective subgradient method by formulating an inverse problem of the optimization problem that is equivalent to WIRL for multi-objective optimizations.
In \cref{sec:WIRL}, we recall the definition of WIRL. 
In \cref{sec:intention_learning_over_inner_product_space}, we adapt the definition of WIRL to multi-objective optimizations. 
In \cref{sec:proof_of_convergence_theorem}, we formulate an inverse problem of an multi-objective optimization problem that is equivalent to the intention learning and show the convergence of WIRL to adapt the inverse problem to the projected subgradient method. 
In \cref{sec:example_of_intention_learning}, we see examples of WIRL applied to linear and quadratic programming.
In \cref{sec:related_work}, with the projected subgradient method,  we show the convergence of MEIRL and GCL.

\section{Wasserstein inverse reinforcement learning}\label{sec:WIRL}

Let $\mathcal{H} , \mathcal{H}_{\mathcal{S}} $ be inner product spaces,
$\mathcal{S} \subset \mathcal{H}_{\mathcal{S}}$ be a space of states
$\mathcal{A} \subset \mathcal{H}$ be a space of actions,
$\mathcal{T} := \prod_k \left( \mathcal{S} \times \mathcal{A} \right)$ be a space of trajectories.
Let $\Theta \subset \mathcal{H}$,
and we call $\Theta$ a space of feature vectors.
Let $\Phi \subset \mathcal{H}$,
and we call $\Phi $ a space of parameters of learner's trajectories.
Let $f_{\bullet } \colon \mathcal{T} \to \Theta$ be 1-Lipschitz,
and we call $f_{\bullet }$ the feature map. 
For any Lipschitz function $r_{\theta} \colon \mathcal{T} \to \mathbb{R}$,
the norm of Lipschitz $\| r_{\theta} \|_L$ is defined by
\[
    \| r_{\theta} \|_L 
    := 
    \sup_{\tau_1 \not = \tau_2 } \frac{ | r_{\theta} (\tau_1) - r_{\theta} (\tau_2) | }{ \| \tau_1 - \tau_2 \| }
    .
\]
Let $\delta_x$ be the Delta function at $x$.
Let ${\{ \tau_{E}^{(n)} \}}_{n=1}^N$ be the data of expert's trajectories,
and we define the distribution of expert's trajectories by
\[
    \mathbb{P}_E := \frac{1}{N} \sum_{n=1}^N \delta_{\tau_{E}^{(n)}}
    .
\]
With the initial state $s_{\mathrm{ini}}^{(n)}$ of expert's trajectories $\tau_{E}^{(n)}$,
and the generator $g_{\bullet} (\bullet ) \colon \Phi \times \mathcal{S} \to \mathcal{T} $, 
we define the distribution of learner's trajectories by
\[
    \mathbb{P}_\phi := \frac{1}{N} \sum_{n=1}^N \delta_{g_{\phi} (s_{\mathrm{ini}}^{(n)})}   
    .
\]
The Wasserstein distance between the distribution $\mathbb{P}_E$ of expert's trajectories and that $\mathbb{P}_\phi$ of learner's trajectories 
is, with the Kantrovich-Rubinstein duality (c.f. \cite{Villani-2009}),
\[
    W ( \mathbb{P}_E , \mathbb{P}_\phi )
    =
    \sup_{\| r_{\theta} \|_L \leq 1 }
    \left\{
        \frac{1}{N} \sum_{n=1}^N r_{\theta} (\tau_{E}^{(n)})
        -
        \frac{1}{N} \sum_{n=1}^N r_{\theta} (g_{\phi} (s_{\mathrm{ini}}^{(n)}))
    \right\}
    ,
\]
where 
$r_{\theta}$ is 1-Lipschitz function.

We are interested in finding $\phi \in \Phi$ satisfying the following problem:
\begin{equation}
    \argmin_{\phi \in \Phi} W ( \mathbb{P}_E , \mathbb{P}_\phi )
    .
    \label{eq:WIRL-origin}
\end{equation}
With 
\begin{equation}\label{eq:sim_1-Lip}
    \left\{ r_{\theta}(\tau) := \theta^{\intercal} f_{\tau} \, \middle| \, \theta \in \Theta \right\}
    \text{ insted of }
    \{ \| r_{\theta} \|_L \leq 1 \}
    ,
\end{equation}
to find $\phi \in \Phi$ satisfying \cref{eq:WIRL-origin}
can be roughly replaced by finding 
\begin{equation}
    \argmin_{\phi \in \Phi} 
    \sup_{ \theta \in \Theta }
    \left\{
        \frac{1}{N} \sum_{n=1}^N \theta^{\intercal} f_{\tau_{E}^{(n)}}
        -
        \frac{1}{N} \sum_{n=1}^N \theta^{\intercal} f_{g_{\phi} (s_{\mathrm{ini}}^{(n)})}
    \right\}
    .
    \label{eq:WIRL-origin-2}
\end{equation}
By changing the sign,
we may consider solving
\begin{equation}
    \argmax_{\phi \in \Phi} 
    \inf_{ \theta \in \Theta }
    \left\{
        \frac{1}{N} \sum_{n=1}^N \theta^{\intercal} f_{g_{\phi} (s_{\mathrm{ini}}^{(n)})}
        -
        \frac{1}{N} \sum_{n=1}^N \theta^{\intercal} f_{\tau_{E}^{(n)}}
    \right\}
    .
    \label{eq:WIRL-origin-3}
\end{equation}
The IRL that solves \cref{eq:WIRL-origin-2} or \cref{eq:WIRL-origin-3},
is called Wasserstein inverse reinforcement learning (WIRL) \cite{Eto-2022}.

\begin{remark}
    In this paper, learning to maximize the reward function of a history-dependent policy is called reinforcement learning.
    Learning that minimizes the score between the reward function calculated from the expert's trajectory and the reward function learned by reinforcement learning is called inverse reinforcement learning.
\end{remark}

\section{WIRL for multi-objective optimization}
\label{sec:intention_learning_over_inner_product_space}

We adapt WIRL to multi-objective optimization (e.g. linear programming, quadratic programming).
Let $\phi \in \Phi$,
$\mathcal{H}^{\prime}$ be an inner product space,
$\mathcal{A}^{\prime}$ be a set such that $\mathcal{A}^{\prime} \subset \mathcal{H}^{\prime}$,
$h \colon \mathcal{A}^{\prime} \to \mathcal{H} $ be a continuous function.
Let $X(s)$ be a compact set\footnote{
If $\mathcal{A}^{\prime}$ is in the Euclid space,
compact sets are bounded closed sets.   
}
in $\mathcal{A}^{\prime}$ for $s \in \mathcal{S}$. 
We set the space of trajectories $\mathcal{T} = \mathcal{S} \times \mathcal{A}$.
Then, multi-objective optimization (e.g. \cites{Murata-1996-multi,Gunantara-2018-review}) is to solve for the following optimization:
\begin{equation}\label{eq:optimal_solver}
    a (\phi , s ) \in \argmax_{h(x) \in h ( X (s))} \phi^{\intercal} h ( x )
    .
\end{equation}
We call the solution or the learner's action $a (\phi , s ) \in \mathcal{A}$ the solver.

We set the feature map $f= \Proj_{\mathcal{A}}$,
where $\Proj_{\mathcal{A}} \colon \mathcal{T} \to \mathcal{A}$ is the projection from $\mathcal{T}$ to $\mathcal{A}$. 
We define the generator $g_{\phi} (s)$ by the following optimization problem:
\begin{equation*}
    g_{\phi} (s) := (s , a (\phi , s ))
    .
\end{equation*}
We say that
\textbf{intention learning} of WIRL
is 
WIRL in the above setting.

The expert's action $a^{(n)}$ is assumed to follow an optimal solution.
Namely, we often run WIRL intention learning by assuming that there exists some $\phi_0 \in \Phi $ and that for any $n$ we can write $a^{(n)} = a (\phi_0 , s^{(n)} )$.

\section{A proof of convergence theorem of intention learning with the projected subgradient method}
\label{sec:proof_of_convergence_theorem}

We explain the projected subgradient method according to \cite{Boyd-2003}.
\begin{definition}\label{defi:projected_subgradient_method}

    Let $\mathcal{H}^{\prime}$ be an inner product space,
    $\ell \colon \mathcal{H}^{\prime} \to \mathbb{R}$ be a convex funciton,
    and
    $\mathcal{C} \subset \mathcal{H}^{\prime}$ be a closed convex set. 
    
    Then, we say \textbf{projected subgradient method} to minimize $\ell$ on $\mathcal{C}$
    is the method to calculate a sequence $\{ \phi_k^{\mathrm{best}} \}_k \subset \mathcal{H}^{\prime}$
    such that for any positive integer $K$,
    \[
        \phi_K^{\mathrm{best}} \in \argmin_{\phi_k \in \{ \phi_k\}_{k=1}^K } \ell (\phi_k)   
        ,
        \quad \phi_{k+1} = \Proj_{\mathcal{C}} ( \phi_k - \alpha_k g_k )
        ,
    \] 
    where 
    $g_k$ is the subgradient\footnote{
        An element $g_k \in \mathcal{H}^{\prime}$ is a subgradient of $\ell$ at $\phi_k$ if and only if
            for arbitrary $\phi \in \mathcal{H}^{\prime} $,
            \[
                \ell (\phi ) \geq \ell ( \phi_k) + (g_k , \phi - \phi_k)
                .    
            \]
        } of $\ell$ at $\phi_k$
    ,
    the sequence $\{ \alpha_k \} \subset \mathbb{R}_{>0}$ be a learning rate
    and $\Proj_{\mathcal{C}}$ it the projection onto $\mathcal{C}$.
\end{definition}
Under the appropriate conditions, the projected subgradient method falls within an error $\varepsilon > 0$ of the minimum of the function $\ell$ in a finite number of iterations.
\begin{proposition}\label{prop:projected_subgradient_method}
    {
        \rm (c.f.
        \cite{Boyd-2003}*{\S 3})
    }
    \begin{description}
        \item[$(1)$] 
            Let $\mathcal{H}^{\prime}$ be an inner product space,
            $\ell \colon \mathcal{H}^{\prime} \to \mathbb{R}$ be the convex function which satisfies
            the Lipschitz condition,
            which means there exists $G > 0 $ such that
            for $\phi , \phi^\prime \in \mathcal{H}^{\prime}$
            \[
                | \ell (\phi ) - \ell (\phi^{\prime}) |
                \leq G \| \phi - \phi^{\prime} \|     
            \]
            and
            $\mathcal{C} \subset \mathcal{H}^{\prime}$ be the closed convex set.
            Let $\{ \alpha_k \}_k \subset \mathbb{R}_{>0}$ be a learning rate.
            We assume that there exists $\phi^* \in \argmin_{\phi \in \mathcal{C}} \ell (\phi)$.

            Then,
            the sequence $\{ \phi_{k}^{\mathrm{best}} \}_k $ which is calculated by the projected subgradient method to minimize $\ell$ on $\mathcal{C}$ satisfies
            \[
                \ell (\phi_k^{\mathrm{best}})
                - \ell (\phi^* )
                \leq 
                \frac{
                    d \left( \phi_1 , \argmin_{\phi \in \mathcal{C}} \ell (\phi) \right)^2
                    +
                    G^2 \sum_{i=1}^k \alpha_i^2
                }{
                    2 \sum_{i=1}^k \alpha_i
                }
                ,    
            \]
            where $d$ is the metric induced by the inner product of $\mathcal{H}^{\prime}$.
        \item[$(2)$] 
            We use the same notation $\mathcal{H}^{\prime}, \ell , \mathcal{C}$ as $(1)$.
            Let $\{ \alpha_k \}_k$ be a nonsummable
            diminishing learning rate, that is,
            \[
                \lim_{k \to \infty } \alpha_k = 0 , 
                \quad
                \sum_{k=1}^{\infty} \alpha_k = \infty
                .    
            \]
            Let $\varepsilon >0$.
            Let $K_1$ be the integer such that 
            for all $k > K_1 $, $\alpha_k \leq \varepsilon / G^2$,
            $K_2$ be the integer such that 
            for all $k > K_2$, 
            \[
                \sum_{i=1}^k \alpha_i
                >
                \frac{1}{\varepsilon}
                \left(
                    d \left( \phi_1 , \argmin_{\phi \in \mathcal{C}} \ell (\phi) \right)^2
                    +
                    G^2
                    \sum_{i=1}^{K_1} \alpha_i^2
                \right)   
            \]
            and 
            $K = \max \{ K_1 , K_2 \}$.
            Then, for $k>K$, we have 
            \[
                \ell (\phi_k^{\mathrm{best}})
                - \ell (\phi^* )
                < \varepsilon
                .
            \]
    \end{description}

\end{proposition}
\cref{prop:projected_subgradient_method} (2) means that
for any error $\varepsilon >0$,
the sequence $\{ \phi_k^{\mathrm{best}} \}$ which is calculated by the projected subgradient method to minimize $\ell $ on $\mathcal{C}$
falls within the error $\varepsilon > 0$ of the minimum of $\ell$ in a finite number of iterations.

\begin{remark}
    In \cite{Boyd-2003}*{\S 3}, they showed \cref{prop:projected_subgradient_method} on $\mathbb{R}^d$.
    Generally, by replacing the argument in \cite{Boyd-2003}*{\S 3} with an inner product space $\mathcal{H}$ instead of $\mathbb{R}^d$,
    we can show 
    \cref{prop:projected_subgradient_method}.
\end{remark}

We give the following inverse problem of optimization problem that is equivalent to the problem handled by intention learning of WIRL.
\begin{definition}\label{defi:intention-learning-IOP}

    Let $\mathcal{H} , \mathcal{H}_{\mathcal{S}}, \mathcal{H}^{\prime} $ be inner product spaces,
    $\mathcal{S} \subset \mathcal{H}_{\mathcal{S}}$ be a space of state,
    $\mathcal{A}^{\prime} \subset \mathcal{H}^{\prime}$,
    $\Phi \subset \mathcal{H}$ be a closed convex set, 
    $h \colon \mathcal{A}^{\prime} \to \mathcal{H} $ be the continuous function,
    $X(s) \subset \mathcal{A}^{\prime}$ be a compact non-empty set for $s \in \mathcal{S}$.

    Then, the inverse problem of multi-objective optimization problem (IMOOP)
    for the solver $a (\phi , s )$ and
    trajectories of an expert $\{ \tau_E^{(n)} = ( s^{(n)}, a^{(n)}) \}_n \subset \mathcal{H}_{\mathcal{S}}\times \mathcal{H}$
    is the problem to find $\phi \in \Phi $ satisfying
    \begin{equation}
        \text{ minimize } 
        \frac{1}{N} \sum_{n=1}^N \phi^{\intercal} a (\phi , s^{(n)} )
            - \frac{1}{N} \sum_{n=1}^N \phi^{\intercal} a^{(n)}
        ,
        \quad
        \text{ subject to }
        \phi \in \Phi
        .
        \label{eq:intention-learning-IOP}
    \end{equation}
\end{definition}

\begin{remark}
    We get the idea of the formulation \cref{defi:intention-learning-IOP} from the formulation of maximal entropy IRL \cite{Ho-Ermon-16-GAIL}. 
    In other words, $\frac{1}{N} \sum_{n=1}^N \phi^{\intercal} a (\phi , s^{(n)} )$ is a reward funciton in reinforcement learning.
\end{remark}

\begin{proposition}\label{prop:IOP_IL_WIRL}
    In the setting of 
    $\Theta = \Phi$,
    \cref{eq:intention-learning-IOP} is
    the replacement of $\max_{\phi \in \Phi}$ and $\inf_{\theta \in \Theta}$
    in \cref{eq:WIRL-origin-3}, that is,
    \begin{align}
        & 
        \min_{\phi \in \Phi}
        \left\{
            \frac{1}{N} \sum_{n=1}^N \phi^{\intercal} a (\phi , s^{(n)} )
            - \frac{1}{N} \sum_{n=1}^N \phi^{\intercal} a^{(n)}
        \right\}
        \notag \\
        & =
        \min_{ \theta \in \Phi }
        \max_{\phi \in \Phi} 
        \left\{
            \frac{1}{N} \sum_{n=1}^N \theta^{\intercal} a (\phi ,s^{(n)})
            -
            \frac{1}{N} \sum_{n=1}^N \theta^{\intercal} a^{(n)}
        \right\}
        .
        \label{eq:IOP_IL_WIRL}
    \end{align}
\end{proposition}

\begin{proof}
    
    By the definition of the solver $a (\phi ,s^{(n)})$, \cref{eq:optimal_solver},
    for any $\theta \in \Phi$
    \[
        \phi^{\intercal} a (\phi ,s^{(n)})
        \geq 
        \phi^{\intercal} a (\theta ,s^{(n)})
        .
    \]
    Since
    \[
        \phi^{\intercal} a (\phi ,s^{(n)})
        \leq 
        \max_{\theta \in \Phi}
        \phi^{\intercal} a (\theta ,s^{(n)})
        ,
    \]
    we see
    \[
        \phi^{\intercal} a (\phi ,s^{(n)})
        =
        \max_{\theta \in \Phi}
        \phi^{\intercal} a (\theta ,s^{(n)})
        .
    \]
    Therefore, we obtain
    \begin{align*}
        &
        \min_{ \theta \in \Phi }
        \max_{\phi \in \Phi} 
        \left\{
            \frac{1}{N} \sum_{n=1}^N \theta^{\intercal} a (\phi ,s^{(n)})
            -
            \frac{1}{N} \sum_{n=1}^N \theta^{\intercal} a^{(n)}
        \right\}
        \\
        & =
        \min_{ \theta \in \Phi } 
        \left\{
            \max_{\phi \in \Phi}
            \frac{1}{N} \sum_{n=1}^N \theta^{\intercal} a (\phi ,s^{(n)})
            -
            \frac{1}{N} \sum_{n=1}^N \theta^{\intercal} a^{(n)}
        \right\}
        \\
        &
        =
        \min_{\phi \in \Phi}
        \left\{
            \frac{1}{N} \sum_{n=1}^N \phi^{\intercal} a (\phi , s^{(n)} )
            - \frac{1}{N} \sum_{n=1}^N \phi^{\intercal} a^{(n)}
        \right\}
        .
    \end{align*}
    It follows
    \cref{eq:IOP_IL_WIRL}.
\end{proof}

\begin{remark}
    It is not guaranteed that the replacement of $\max_{\phi \in \Phi}$ and $\inf_{\theta \in \Theta}$ in \cref{eq:WIRL-origin-3} coincides with \cref{eq:WIRL-origin-3}.
\end{remark}

By \cref{prop:IOP_IL_WIRL}, we can interpret Intention learning (of WIRL) as solving the IMOOP.
We can also show that intention learning converges with \cref{prop:projected_subgradient_method}.
We explain these.

We set
\begin{equation*}
    F ( \phi ) 
    :=
        \frac{1}{N} \sum_{n=1}^N \phi^{\intercal} a (\phi , s^{(n)} )
        - \frac{1}{N} \sum_{n=1}^N \phi^{\intercal} a^{(n)}
    .
\end{equation*}
To adapt \cref{prop:projected_subgradient_method} to $\Phi$ and $F$,
we show the following lemma:
\begin{lemma}\label{lem:WIRL-projected_subgradient_method}

    In the setting of \cref{defi:intention-learning-IOP},
    \begin{enumerate}
        \item[$(1)$] the function $F$ is convex, 
        \item[$(2)$] the fuction $F$ is Lipschitz,
        \item[$(3)$] one of the subgradient of $F$ at $\phi \in \Phi$ is
        $
            \frac{1}{N} \sum_{n=1}^N a (\phi , s^{(n)} )
            - \frac{1}{N} \sum_{n=1}^N a^{(n)}
            .
        $
    \end{enumerate}
\end{lemma}
\begin{proof}
    We note that by \cref{eq:optimal_solver}, for $n$,
    \[
        \phi^{\intercal} a (\phi , s^{(n)} )
    = 
        \max_{x \in X(s^{(n)})} \phi^{\intercal} h (x )
        .
    \]
    $(1)$ 
    Since $X(s^{(n)})$ is compact and
    the map $h \colon \mathcal{A}^{\prime} \to \mathcal{H}$ and 
    for $\phi \in \Phi$ the map $\phi^{\intercal} \colon \mathcal{H} \to \mathbb{R} ; \mathcal{H} \ni y \mapsto \phi^{\intercal} y $
    are continuous,
    we note that
    there exists the maximal of
    $ \phi^{\intercal} h( \bullet ) $ on $X(s^{(n)})$.
    Let $\phi_1 , \phi_2 \in \Phi$.
    For any $x \in X(s^{(n)})$, we see
    \[
        t \phi_1^{\intercal} h(x ) + (1-t) \phi_2^{\intercal} h(x) 
        \geq (t \phi_1 + (1-t ) \phi_2)^{\intercal} h(x)    
        .
    \]
    By applying $\max_{x \in X(s^{(n)})}$ to the first and second terms on the left-hand side,
    for any $x \in X(s^{(n)})$, we obtain
    \[
        t \max_{x \in X(s^{(n)})} \phi_1^{\intercal} h(x ) 
        + (1-t) \max_{x \in X(s^{(n)})} \phi_2^{\intercal} h(x) 
        \geq (t \phi_1 + (1-t ) \phi_2)^{\intercal} h(x)   
        . 
    \]
    By applying $\max_{x \in X(s^{(n)})}$ to the right-hand side,
    \[
        t \max_{x \in X(s^{(n)})} \phi_1^{\intercal} h(x ) 
        + (1-t) \max_{x \in X(s^{(n)})} \phi_2^{\intercal} h(x) 
        \geq \max_{x \in X(s^{(n)})} (t \phi_1 + (1-t ) \phi_2)^{\intercal} h(x)    
        .
    \]
    Therefore, $\Phi \ni \phi \mapsto \max_{x \in X(s^{(n)})} \phi^{\intercal} h(x ) = \phi^{\intercal} a (\phi , s^{(n)} )$ is convex function.
    Since the sum of convex functions is convex,
    $F ( \phi ) $ is convex.

    $(2)$
    For any $x \in X(s^{(n)})$, by Cauchy-Schwarz' inequality,
    \begin{equation}
        \phi_1^{\intercal} h ( x ) - \phi_2^{\intercal} h( x) \leq \| \phi_1 - \phi_2 \|  \| h( x ) \|
        .
        \label{eq:Lip-CS}
    \end{equation}
    Since $X(s^{(n)})$ is compact and
    the map $h \colon \mathcal{A}^{\prime} \to \mathcal{H}$ and 
    the map $\| \bullet \| \colon \mathcal{H} \to \mathbb{R} ; \mathcal{H} \ni y \mapsto \| y \|$
    are continuous,
    we note that
    there exists the maximal of
    $ \| h( \bullet ) \| $ on $X(s^{(n)})$.
    To adapt $\max_{x \in X(s^{(n)}) }$ the right-hand side of \cref{eq:Lip-CS},
    for $x \in X(s^{(n)})$,
    \[
        \phi_1^{\intercal} h ( x ) - \phi_2^{\intercal} h( x) \leq \| \phi_1 - \phi_2 \| \max_{x^{\prime} \in X(s^{(n)}) } \| h( x^{\prime} ) \|
        .
    \]
    To adapt $\max_{x \in X(s^{(n)}) }$ to the second term of the left-hand side,
    for $x \in X(s^{(n)})$,
    \[
        \phi_1^{\intercal} h ( x ) - \max_{x^{\prime} \in X(s^{(n)}) } \phi_2^{\intercal} h( x^{\prime} ) \leq \| \phi_1 - \phi_2 \| \max_{x^{\prime} \in X(s^{(n)}) } \| h( x^{\prime} ) \|
        .
    \]
    To adapt $\max_{x \in X(s^{(n)}) }$ to the first term of the left-hand side,
    \[
        \max_{x \in X(s^{(n)}) } \phi_1^{\intercal} h ( x ) 
        - \max_{x \in X(s^{(n)}) } \phi_2^{\intercal} h( x ) 
        \leq \| \phi_1 - \phi_2 \| \max_{x \in X(s^{(n)}) } \| h( x ) \|
        .
    \]
    Since the above inequation also holds if we swap $\phi_1, \phi_2$, 
    \[
        \left| 
            \max_{x\in X(s^{(n)}) } \phi_1^{\intercal} h ( x ) 
            - \max_{x\in X(s^{(n)}) } \phi_2^{\intercal} h( x ) 
        \right|
        \leq 
            \| \phi_1 - \phi_2 \| \max_{x \in X(s^{(n)}) } \| h( x ) \|
        .
    \]
    It meas that $\phi \mapsto \max_{x \in X(s^{(n)})} \phi^{\intercal} h(x ) = \phi^{\intercal} a (\phi , s^{(n)} )$ is Lipschitz continuous.
    Since the sum of Lipschitz functions is Lipschitz,
    $F ( \phi ) $ is Lipschitz.

    $(3)$ For any $\phi_1, \phi_2 \in \Phi$,
    we have 
    \begin{equation*}
        \phi_2^{\intercal} \left( a(\phi_1 , s^{(n)}) - a^{(n)} \right)
        = \phi_1^{\intercal} \left( a(\phi_1 , s^{(n)}) - a^{(n)} \right)
        + (\phi_2 - \phi_1 )^{\intercal} \left( a(\phi_1,  s^{(n)}) - a^{(n)} \right)   
        .
    \end{equation*}
    By the definition of the solver $a(\phi_1 , s^{(n)})$,
    \[
        \phi_2^{\intercal} \left( a(\phi_2 , s^{(n)}) - a^{(n)} \right)
        \geq 
        \phi_2^{\intercal} \left( a(\phi_1 , s^{(n)}) - a^{(n)} \right)   
        . 
    \]
    Therefore, we see
    \begin{equation*}
        \phi_2^{\intercal} \left( a(\phi_2 , s^{(n)}) - a^{(n)} \right)
        \geq 
        \phi_1^{\intercal} \left( a(\phi_1 , s^{(n)}) - a^{(n)} \right)
        + (\phi_2 - \phi_1 )^{\intercal} \left( a(\phi_1,  s^{(n)}) - a^{(n)} \right)  
        . 
    \end{equation*}
    Taking the average of both sides for $n$,
    \[
        F ( \phi_2 )
        \geq
        F ( \phi_1 )
        +
        (\phi_2 - \phi_1 )^{\intercal} 
        \left( 
            \frac{1}{N} \sum_{n=1}^N 
                a(\phi_1,  s^{(n)}) 
            - 
            \frac{1}{N} \sum_{n=1}^N
                a^{(n)} 
        \right)  
        . 
    \]
\end{proof}

\begin{remark}
    Barmann et al. showed \cref{lem:WIRL-projected_subgradient_method} (1) and (3) for linear programming \cite{Barmann-2018-online}*{Proposition 3.1}
\end{remark}

The algorithm of WIRL for multi-objective optimization
is given by \cref{alg:intention-WIRL-gradual-decay}.
\begin{figure}[ht]
    \begin{algorithm}[H]
        \caption{Intention learning (of WIRL)}\label{alg:intention-WIRL-gradual-decay}
        \begin{algorithmic}[1] 
            \STATE initialize $ \phi_1 \in \Phi$
            \FOR{$k =1 , \ldots , K-1$}
                \STATE $\phi_{k+1} \leftarrow \phi_k 
                - \frac{\alpha_k}{N} \sum_{n=1}^N 
                \left(
                    a ( \phi_k , s^{(n)} )
                    -
                    a^{(n)}
                \right)
                $
                \STATE projection onto $\Phi$ for $\phi_{k+1}$
            \ENDFOR
            \RETURN $ \phi_K^{\mathrm{best}} \in \argmin_{\phi_k \in \{ \phi_k \}_{k=1}^K} F(\phi_k)$
        \end{algorithmic}
    \end{algorithm}
\end{figure}
\begin{remark}
    In
    \cite{Eto-2022},
    no operation ``projection onto $\Phi$ for $\phi_k$'' is performed
    on \cref{alg:intention-WIRL-gradual-decay}.
    We add this operation on \cref{alg:intention-WIRL-gradual-decay}
    to discuss the intention learning 
    with the projection onto $\Phi$
    in
    \cref{sec:proof_of_convergence_theorem}.
\end{remark}

\begin{lemma}\label{lem:IOP_coninsides_with_Intension_L}
    In the setting of \cref{defi:intention-learning-IOP},
    the algorithm which solves IMOOP for the solver $a (\phi , s )$
    coninsides with \cref{alg:intention-WIRL-gradual-decay}.
\end{lemma}
\begin{proof}
    By \cref{defi:projected_subgradient_method} and \cref{lem:WIRL-projected_subgradient_method} (3),
    the projected subgradient method to minimize $F$ on $\Phi$
    coincides with
    \cref{alg:intention-WIRL-gradual-decay}.
\end{proof}

Since \cref{prop:projected_subgradient_method},
\cref{lem:WIRL-projected_subgradient_method,lem:IOP_coninsides_with_Intension_L},
intention learning of WIRL, \cref{alg:intention-WIRL-gradual-decay} is convergent.
\begin{theorem}\label{theo:WIRL-convergence-theorem}

    Let $\mathcal{H} , \mathcal{H}_{\mathcal{S}}, \mathcal{H}^{\prime} $ be inner product spaces,
    $\mathcal{S} \subset \mathcal{H}_{\mathcal{S}}$, 
    $\mathcal{A}^{\prime} \subset \mathcal{H}^{\prime}$, 
    $\Phi \subset \mathcal{H}$ be a closed convex set,
    $h \colon \mathcal{A}^{\prime} \to \mathcal{H} $ be a continuous function,
    $X(s) \subset \mathcal{A}^{\prime}$ be the compact non-empty set for $s \in \mathcal{S}$.
    Let $\{ \alpha_k \}_k \subset \mathbb{R}_{>0}$ be a nonsummable diminishing learning rate,
    that is,
        \[
            \lim_{k \to \infty } \alpha_k = 0 , 
            \quad
            \sum_{k=1}^{\infty} \alpha_k = \infty    
            .
        \]
    Assume that there exists the minimum of $F$ on $\Phi$.

    Then, for any $\varepsilon > 0$, the sequence $\{ \phi_k^{\mathrm{best}} \}_k$ which is calculated by intention learning of WIRL for any error $\varepsilon >0 $,
    there exists a positive integer $K$ such that 
    for all $k > K$,
    \[
        F(\phi_k^{\mathrm{best}}) - \min_{\phi \in \Phi } F (\phi ) < \varepsilon
        .    
    \]

\end{theorem}
It means that for any error $\varepsilon >0$,
intention learning of WIRL falls within the error $\varepsilon$
of the minimum of $F$ in a finite number of iterations.

\begin{remark}\label{remark:intention_for_constant_step}
    By \cref{prop:projected_subgradient_method} (1)
    if the learning rate $\{ \alpha_k \}_k \subset \mathbb{R}_{>0}$ is constant $\alpha$,
    \cref{alg:intention-WIRL-gradual-decay}
    falls within an error
    \[
        \frac{\alpha}{2N} \sum_{n=1}^N \max_{x^{\prime} \in X(s^{(n)}) } 
        \| 
            h( x^{\prime} ) -a^{(n)}
        \| 
    \]
    of
    the minimum of $F$ on $\Phi$ in a finite number of iterations.
    However,
    it is not guaranteed
    that for any $\varepsilon > 0$
    \cref{alg:intention-WIRL-gradual-decay}
    falls within the error $\varepsilon $ of
    the minimum of $F$ on $\Phi$ in a finite number of iterations.
\end{remark}

\begin{remark}
    \cite{Suzuki-2019-TV}*{Algorithm 2} coincides with \cref{alg:intention-WIRL-gradual-decay} for $0$-$1$ planning problem, $\Phi = \mathbb{R}^n$ and a constant learning rate. However, by \cref{remark:intention_for_constant_step}, it is not guaranteed that for any $\varepsilon > 0$ \cite{Suzuki-2019-TV}*{Algorithm 2} falls within the error $\varepsilon $ of the minimum of $F$ in a finite number of iterations. If we take a learning rate satisfying nonsummable diminishing, it is guaranteed that for any $\varepsilon > 0$ \cite{Suzuki-2019-TV}*{Algorithm 2} falls within the error $\varepsilon $ of the minimum.
\end{remark}

\section{Examples}\label{sec:example_of_intention_learning}

\begin{example}[mixed integer linear programming]\label{example:mip}
    We set $\mathcal{A} = \mathbb{R}^d$,
    \[
        \Phi = \Delta^d ,
        \quad
        \Delta^d := \left\{ \phi = (\phi_1 , \ldots , \phi_d ) \in \mathbb{R}^d \, \middle| \, \forall i , \, \phi_i \geq 0 , \,  \sum_{i=1}^d \phi_i = 1 \right\}, 
    \]
    and 
    the vector valued function $h = \Id \colon \mathbb{R}^d \to \mathbb{R}^d $.
    We assume that $X (s)$ is the finite direct sum of bounded convex polyhedrons
    for any $s \in \mathcal{S}$.
    In the above setting,
    we can do intention learning of WIRL.

    The action, i.e. solver is given by 
    \[
        a (\phi , s ) \in \argmax_{x \in X (s)} \phi^{\intercal} x 
        .
    \]
    It means that $a(\phi , s )$ is a solver of a mixed integer linear programming.

    When we do 
    \cref{alg:intention-WIRL-gradual-decay},
    we use \cite{Wang-13-projection}
    as the algorithm to implement 
    $\Proj_{\Delta^d}$.
    Since $\Phi$ is compact, there exists the minimum of $F$ on $\Phi$, and we can adapt \cref{theo:WIRL-convergence-theorem} to this case.

\end{example}

\begin{remark}\label{remark:mip}
    In \cref{example:mip},
    the space on we research is
    \[
        \tilde{\Phi} = \left\{ \phi = (\phi_1 , \ldots , \phi_d ) \in \mathbb{R}^d \, \middle| \, \forall i , \, \phi_i \geq 0  \right\} \setminus \{ (0 , \ldots , 0 ) \}   .
    \]
    However, it is more suitable to research 
    $
        \Phi = \Delta^d
    $
    than $\tilde{\Phi}$.
    There are two reason.

    First,
    there are cases where there is no minimum on $\tilde{\Phi}$.
    Since 
    for arbitrary $\gamma >0 $, $\phi \in \tilde{\Phi}$,
    \begin{equation}\label{eq:mip-remark}
        F (\gamma \phi ) = \gamma F (\phi )  
    \end{equation}
    and
    \[
        F( 0 , \ldots , 0  ) = 0    , 
    \]
    if $F(\phi ) > 0$ for any $\phi \in \tilde{\Phi}$,
    there is no minimum on $\tilde{\Phi}$.
    In addition, 
    If there exists $\phi \in \tilde{\Phi}$ such that $F(\phi ) < 0$,
    by \cref{eq:mip-remark}
    to take a large enough $\gamma$,
    we can move $F( \gamma \phi ) $ close to $-\infty$.
    It means there is no minimum on $\tilde{\Phi}$ if there exists $\phi \in \tilde{\Phi}$ such that $F(\phi ) < 0$.

    Second,
    it is enough to research $\Phi$.
    The solver $a(\phi , s^{(n)})$ is invariant under scalar multiplication
    and $\tilde{\Phi} = \{ c \phi \,| \,  c \in \mathbb{R}_{>0} , \, \phi \in \Delta^d \}$.
    Therefore it is enough to research $\Delta^d$, the set of equivalence classes.

\end{remark}

\begin{example}[mixed integer quadratic programming]\label{exa:QP-1}
    We set 
    the space $\mathcal{A} = \mathbb{S}^d \times \mathbb{R}^{d}$,
    and
    $
        \Phi = \Gamma^d \times \square^{d} (b_0) 
        ,
    $
    where $\mathbb{S}^d$ is the space of symmetric matrixes of order $d$,
    $\mathbb{S}_+^d$ is the space of positive semidefinite matrixes of order $d$,
    and 
    \begin{align*}
        \Gamma^d :=
        \left\{
            A \in \mathbb{S}_+^d
        \, \middle| \,
            \Tr (A ) = 1
        \right\}    
        ,
        \quad
        \square^d (b_0) :=
        \prod_{i=1}^d \{ b_i \in \mathbb{R} \, | \, - b_0 \leq b_i \leq b_0 \} 
        \text{ for } b_0 >0.
    \end{align*}
    For any square matrix $A$ of order $d$,
    $b \in \mathbb{R}^{d}$
    we set $\phi = (A, b)$,
    and 
    set the vector-valued function $h \colon \mathbb{R}^d \to \mathbb{S}^d \times \mathbb{R}^{d} $ by  
    for $x \in \mathbb{R}^{d}$,
    \[
        \phi^{\intercal} h (x)
        := - x^{\intercal} A x - b^{\intercal} x  
        . 
    \]
    We assume that $X (s)$ is the finite direct sum of bounded convex polyhedrons
    for any $s \in \mathcal{S}$.
    In the above setting,
    since $\Phi$ is convex,
    we can do intention learning of WIRL.

    The action, i.e. solver is given by
    \[
        a (\phi , s ) \in \argmax_{h(x) \in h( X (s))}  ( - x^{\intercal} A x - b^{\intercal} x ) 
        .
    \]
    It means $a(\phi , s )$ is a solver of 
    concave mixed integer quadratic programming.

    When we do 
    \cref{alg:intention-WIRL-gradual-decay},
    we use
    \[
        \Proj_{\Phi} (A,b)
        = 
        \left( 
            \Proj_{\Gamma^d} (A)
            , \Proj_{\square^d (b_0)} (b)
        \right)
    \]
    as the algorithm to implement 
    $\Proj_{\Phi}$.
    Here $\Proj_{\Gamma^d} (A)$ is given by 
    $\Proj_{\Gamma^d} (A) = \sum_{i=1}^d \mu_i v_i v_i^{\intercal}$
    where $A = \sum_{i=1}^d \lambda_i v_i v_i^{\intercal}$ is the eigenvalue decomposition of $A$,
    $\lambda = (\lambda_1 , \ldots , \lambda_d)$,
    and $\mu = (\mu_1 , \ldots, \mu_d ) = \Proj_{\Delta^d} (\lambda )$
    \cite{Beck-2017-First}*{Example 7.20},
    and $\Proj_{\square^d (b_0)} (b)$ is given by 
    \begin{align*}
        \Proj_{\square^d (b_0)} (b) & = (\Proj_{\square^1 (b_0)} (b_1) , \ldots , \Proj_{\square^1 (b_0)} (b_d) ),
        \\
        \Proj_{\square^1 (b_0)} (b_i) & = \max \{ \min \{ b_i , b_0 \} , - b_0 \} ,
        \,
        i=1 , \ldots , d 
        .
    \end{align*}

    In the same way as \cref{remark:mip},
    $\Phi$ does not include the pair of the zero matrix and the zero vector.
    It means that
    \cref{alg:intention-WIRL-gradual-decay} is not convergence at this pair.
    Since $\Phi$ is compact, there exists the minimum of $F$ on $\Phi$, and we can adapt \cref{theo:WIRL-convergence-theorem} to this case.
    
\end{example}

\section{Related work}\label{sec:related_work}

\subsection*{Convergence of Inverse reinforcement learning}

We believe that various inverse reinforcement learning convergence properties can be guaranteed with the gradient descent (e.g. \cite{Garrigos-2023-handbook}*{Theorem 3.4}) for $L$-smooth fucntion including smooth funcitons and projected subgradient method \cref{prop:projected_subgradient_method} for Lipschitz functions. 
Specifically, we give the following examples.

In MEIRL \cite{Ziebart-2008-maximum},
the objective function $L$ is convex.
In addition, if trajectories is finite, $L$ is smoooth.
We assume that there exists the maximum of $L$ on the space of weight of reward (here denote $\Theta$).
For example, if $\Theta$ is compact, $L$ has mamimum on $\Theta$.
By gradient descent, that is, \cite{Garrigos-2023-handbook}*{Theorem 3.4}, 
we can approximate the maximum reward weight $\theta$ in $L$.
This means that the convergence of MEIRL is guaranteed.

In relative entropy inverse reinforcement learning (REIRL),
the dual objective fucntion $g$ is concave \cite{Boularias-2011-relative}.
In addition, if trajectories is finite, $g$ is Lipschitz.
We assume that there exists the maximum of $g$ on the space of weight of reward (here denote $\Theta$).
By projected subgradient method, that is, \cref{prop:projected_subgradient_method}, 
we can approximate the maximum reward weight $\theta$ in $g$.
However, since we are approximating the gradient of $g$ with respect to $\theta$ using weighted sampling, there is room to consider whether convergence of REIRL will actually occur.

In GCL \cite{Finn-2016-guided},
the importance sampling of the objective fucntion ${\mathcal{L}}_{\mathrm{IOC}}$ (here denote $L$) is convex for $\theta$.
In addition, if trajectories is finite, $L$ is Lipschitz.
We assume that there exists the maximum of $L$ on the space of weight of reward (here denote $\Theta$).
By gradient descent, that is, \cite{Garrigos-2023-handbook}*{Theorem 3.4}, 
we can approximate the maximum reward weight $\theta$ in $L$.
This means that the convergence of GCL is guaranteed.

\subsection*{WGAN and WIRL}
We explain why WGAN does not converge, while WIRL does.
In WGAN, in general, parameterizing the class of 1-Lipschitz functions does not necessarily make the Critic convex.
In WIRL, on the other hand, parameterizing the class of 1-Lipschitz functions with \cref{eq:sim_1-Lip} and \cref{prop:IOP_IL_WIRL} make the Critic, i.e. the contents of $\sup$ in \cref{eq:WIRL-origin-2} convex.
In other words, 
\cref{lem:WIRL-projected_subgradient_method} (2) follows.

\section{Conclusion}
We proved the convergence of WIRL for multi-objective optimizations with the projective subgradient method by formulating an IMOOP that is equivalent to WIRL for multi-objective optimizations.
In other words, 
we showed WIRL is convergent at the minimum of $F$ on $\Phi$.

We raise some future works.
We have proved that WIRL converges, whereas it has not been proven whether learner's trajectories mimic expert's trajectories when WIRL is convergent.
If we show this, we can say that WIRL theoretically converges in the direction that learner's trajectories imitate expert's trajectories.
This means that WIRL is theoretically guaranteed to have a mechanism that frees users from manually designing the reward function.

\begin{ack}

The authors would like to thank Kei Takemura for his valuable comments.
They
would also like to thank Shinji Ito for carefully reading this paper.

\end{ack}

\bibliography{WIRL} 


\end{document}